\documentclass[conference]{IEEEtran}
\usepackage{fullpage}
\usepackage[margin=0.7in]{geometry}
\usepackage{graphicx}
\usepackage{cite}
\usepackage{bbm}
\usepackage[cmex10]{amsmath}
\usepackage{algorithmic}
\usepackage[]{algorithm2e}
\usepackage[T1]{fontenc}
\usepackage[scaled]{helvet}
\usepackage{amssymb}
\usepackage{epstopdf}
\usepackage{float}
\usepackage{subcaption}
\usepackage{array}
\usepackage{amsthm}
\theoremstyle{plain}
\newtheorem{theorem}{Theorem}
\newtheorem{definition}{Definition}
\newtheorem{lemma}{Lemma}
\newtheorem{example}{Example}

\begin{document}
\title{Sequential Logistic Principal Component Analysis (SLPCA): Dimensional Reduction in Streaming Multivariate Binary-State System}

\author{\IEEEauthorblockN{Zhaoyi Kang, Costas J. Spanos}
\IEEEauthorblockN{Dept. of Electrical Engineering \& Computer Sciences, UC Berkeley, Berkeley, CA 94709\\
\{kangzy, spanos\}@berkeley.edu, 
}
}

\maketitle

\begin{abstract}
Sequential or online dimensional reduction is of interests due to the explosion of streaming data based applications and the requirement of adaptive statistical modeling, in many emerging fields, such as the modeling of energy end-use profile. Principal Component Analysis (PCA), is the classical way of dimensional reduction. However, traditional Singular Value Decomposition (SVD) based PCA fails to model data which largely deviates from Gaussian distribution. The Bregman Divergence was recently introduced to achieve a generalized PCA framework. If the random variable under dimensional reduction follows Bernoulli distribution, which occurs in many emerging fields, the generalized PCA is called Logistic PCA (LPCA)~\cite{schein2003generalized}. In this paper, we extend the batch LPCA to a sequential version (i.e. SLPCA), based on the sequential convex optimization theory. The convergence property of this algorithm is discussed compared to the batch version of LPCA (i.e. BLPCA), as well as its performance in reducing the dimension for multivariate binary-state systems. Its application in building energy end-use profile modeling is also investigated.   

\end{abstract}


%
\IEEEpeerreviewmaketitle

\section{Introduction}
\IEEEPARstart
Sequential data mining has received considerable attention recently as the development in wireless-sensor information technology facilitates the collection of huge amount of streaming data -- This brings about several challenges on the efficiency in computation, storage and statistical learning ~\cite{hastie2009elements}. Dimensional reduction in the streaming environment is one of the techniques that can help to overcome those issues~\cite{papadimitriou2005streaming}.

Among the dimensional reduction techniques, Principal Component Analysis (PCA) is most widely-known. PCA finds the linear projection of the original data matrix which explains the largest portion of the variance. From the maximum likelihood perspective, PCA contains the assumption that the data follows Gaussian distribution. Naturally, this will fail to give reliable results when data largely deviates from Gaussian distribution~\cite{tipping1999probabilistic}. Bregman Divergence is introduced to achieve a generalized PCA framework for a family of exponential distributed data (i.e. \emph{e}PCA)~\cite{collins2001generalization}. As a generalization over the Frobenious norm, KL-divergence, Mahalanobis distance etc., Bregman Divergence is believed to better quantify the distance of variables coming from non-Gaussian distribution~\cite{banerjee2005clustering}~\cite{bregman1967relaxation}. In the case of Bernoulli random variables, which we are interested in, the generalized PCA can be viewed as Logistic PCA (LPCA). 

In this work, we extend the LPCA to the sequential version, based on the sequential convex optimization theory~\cite{zinkevich2003online}~\cite{shalev2011online}. The convergence property of this algorithm is discussed with respect to the batch optimization algorithm. An application in building energy end-use profile modeling is investigated as an experiment of this method, which demonstrates its capability in reducing dimension in multivariate binary-state systems. 

This paper is organized as follows: In Section II, the background and the detail of the algorithm is given, including PCA, exponential family, the Bregman Divergence and eventually the sequential LPCA (i.e. SLPCA) which we propose. In Section III, the convergence property of the algorithm is discussed, followed by the simulation results as well as the application in energy end-use modeling in Section IV. In Section V, conclusion is drawn. 

\section{Algorithm Framework}

PCA as a dimensional reduction technique has been well studied, and our Sequential LPCA is essentially a generalized incremental version of the classical model. 

\subsection{Principal Component Analysis}

PCA is a well-known technique for dimensional reduction for high dimension data. It is of special importance in high dimensional regression model, and in a variety of applications, ranging from face recognition to generalized machine learning~\cite{vidal2005generalized}~\cite{hastie2009elements}.
 
There are two perspectives of PCA~\cite{tipping1999probabilistic}. The first is the matrix factorization perspective. For a matrix $\mathbf{X} \in \mathbb{R}^{N \times P}$ , we find a lower rank matrix $\mathbf{\Theta}$ to minimize the error:

\begin{equation}\label{eq:1}
\min_{\mathbf{\Theta}} \| \mathbf{X} - \mathbf{\Theta} \|_{F}^2
\end{equation}

\hangindent=0em
\hangafter=0
in which $\|\cdot\|_{F}$ is the Frobenious norm. This problem can be solved by Singular Vector Decomposition (SVD).

\begin{definition}[Singular Value Decomposition (SVD)]
for input data matrix $\mathbf{X} \in \mathbb{R}^{N \times P}$, Singular Value Decomposition decomposes the matrix to be:
\begin{equation}\label{eq:2}
\mathbf{X} = \mathbf{U} \mathbf{\Sigma} \mathbf{V}^T = \sum_{i=1}^{\min(N,P)} \sigma_i \mathbf{u}_i \mathbf{v}_i^{T}
\end{equation}
in which $\mathbf{U}$ is an $N \times N$ matrix called column eigenvector matrix, with $\mathbf{u}_i$ as $i^{th}$ columns; $\mathbf{V} \in \mathbb{R}^{P \times P}$ called row eigenvector matrix, with $\mathbf{v}_i$ as $i^{th}$ columns; $\mathbf{\Sigma} \in \mathbb{R}^{N \times P}$ rectangular diagonal matrix, with $i^{th}$ diagonal value as $\sigma_i$.
\end{definition}

If $\sigma_i$'s are sorted and the largest is $\sigma_1$, then $\mathbf{\Theta} = \sigma_1 \mathbf{u}_1 \mathbf{v}_1^{T}$ is the solution to Equation \eqref{eq:1} if rank$(\mathbf{\Theta})=1$. If the rank$(\mathbf{\Theta})=r$, then we choose $\mathbf{\Theta} = \sum_{i=1}^r \sigma_i \mathbf{u}_i \mathbf{v}_i^{T}$.

However, there is another perspective of PCA that is less widely-known, which is called the probabilistic interpretation. Here, the columns of $\mathbf{X} \in \mathbb{R}^{N \times P}$ can be viewed as $N$ samples drawn from a Gaussian distribution with dimension lower than $P$. This idea can be used in larger family of distributions, for example, the exponential family distributions.

\subsection{Exponential Family}
\begin{definition}[Exponential Family]
In the exponential family of distributions the conditional probability of a value $x$ given parameter value $\theta$ takes the following form~\cite{casella1990statistical}:
\begin{equation}\label{eq:4}
\log P (x|\theta) = \log P_{0} (x) + x\theta - G(\theta)
\end{equation}
In which, $\theta$ is the natural parameter of the distribution. $G(\theta)$ is a function that ensures that the sum (integral) of $P(x|\theta)$ over the domain of $x$ is one. It is observed that $G(\theta) = \log \sum_{x} P_0 (x) e^{x\theta}$.
\end{definition}
Borrowing idea from the probabilistic view of PCA, if $x$ is the original data, $\theta$ comes from a lower dimensional space.

\subsection{Exponential Family PCA}

Equation \eqref{eq:1} becomes inappropriate when the data is not Gaussian, which happens a lot in real world. Instead of Frobenious norm, we need another way to quantify the distance between and its lower rank approximation. The Bregman Divergence is introduced to generalize the distance quantification~\cite{bregman1967relaxation}~\cite{banerjee2005clustering}~\cite{collins2001generalization}.

\begin{definition}[Bregman Divergence]
The \emph{Bregman divergence} w.r.t. $F$ is defined, for $p, q \in \mathbb{R}^d$, as:
\begin{equation}\label{eq:5}
B_{F}(p \| q) = F(p) - ( F(q) + \nabla F(x)^{T} \cdot (p-q))
\end{equation}
For an exponential family distribution in (4), let $F(g(\theta)) + G(\theta) = g(\theta)\theta$, in which $g(x) = \nabla G(x)$. If $\mathbf{P}, \mathbf{Q}$ are matrices, $B_{F}(\mathbf{P} \| \mathbf{Q}) = \sum_{i,j} B_{F} (\mathbf{P}_{ij} \| \mathbf{Q}_{ij})$. 
\end{definition}

\begin{example}
In the case of Gaussian distribution, the Bregman Divergence equals to squared loss $B(x\|g(\theta))=\frac{1}{2}(x-\theta)^2$.
\end{example}
\begin{example}
In the case of Bernoulli distribution, Bregman Divergence is the logit function $B(x\|g(\theta))=\log(1+\exp(-x^*\theta))$, in which $x^*$ is a transformation of $x$ as $x^*=2x-1 \in \{-1,1\}$. In this case, Bregman Divergence is a convex function of $\theta$, thus can be placed in an efficient optimization framework. 
\end{example}

Therefore, similar to Equation \eqref{eq:1}, we can construct an optimization problem based on the Bregman Divergence. For data matrix $\mathbf{X}$ and $[\mathbf{X}]_{ij} = x_{ij}$, we use $g(\mathbf{\Theta})$ to approximate, i.e. $[\mathbf{\Theta}]_{ij}=\theta_{ij}$:
\begin{equation}\label{eq:6}
\min_{\mathbf{\Theta}} \sum_{i,j} B(x_{ij}\|g(\theta_{ij}))
\end{equation}
In this work, we mainly focus on Bernoulli random variables, in other words, the Logistic PCA (LPCA). Hence, the optimization problem is as in Example 2. The logit function in some cases is not strictly convex, thus we need to regularize the $\mathbf{\theta}$ variable by $L(\mathbf{\theta})$:
\begin{equation}\label{eq:7}
\min_{\mathbf{\Theta}} \sum_{i,j} B(x_{ij}\|g(\theta_{ij})) + L(\mathbf{\Theta})
\end{equation}
If we want to optimize Equation \eqref{eq:7} with the constraint of the rank of $\mathbf{\Theta}$, we can re-write $\mathbf{\Theta}$ as $\mathbf{A} \mathbf{V}^T$ where $\mathbf{A} \in \mathbb{R}^{N \times r}$ and $\mathbf{V} \in \mathbb{R}^{P \times r}$ s.t. rank$(\mathbf{A} \mathbf{V}^T) = r$. Then, we will minimize over two matrices $\mathbf{A}$ and $\mathbf{V}$. 
\begin{equation}\label{eq:8}
\min_{\mathbf{A}\in \mathbb{R}^{N \times r},\mathbf{V}\in \mathbb{R}^{P \times r}} \sum_{i,j} B(\mathbf{X}\|g(\mathbf{A} \mathbf{V}^T)) + \gamma \Gamma(\mathbf{A}) + \lambda R(\mathbf{V})
\end{equation}
\hangindent=0em
\hangafter=0
where $\Gamma(\cdot)$ and $R(\cdot)$ are regularization functions. In our work, both functions are quadratic, $\Gamma(\cdot)=R(\cdot)=\frac{\|\cdot\|^2}{2}$.

\subsection{Batch Logistic PCA (BLPCA)}
For the optimization problem in Equation \eqref{eq:8}, we can solve it in an alternating minimization algorithm~\cite{collins2001generalization}~\cite{csisz1984information}. If we define $\mathbf{a}_t$ as the $t^{th}$ row of $\mathbf{A}$, let:
\begin{equation}\label{eq:9}
h_t(\mathbf{a}_t, \mathbf{V}) = B(\mathbf{x}_t\|g(\mathbf{a}_t \mathbf{V}^T)), t = 1,\cdots, N
\end{equation}
Then we can solve Equation \eqref{eq:8} by iterating the following two steps, and we call this method Batch LPCA (BLPCA):
\begin{align}
\mathbf{a}_t^* & = arg \min_{\mathbf{a} \in \mathbb{R}^{1 \times r}} h_t(\mathbf{a}, \mathbf{V}^*) + \gamma \frac{\|\mathbf{a}\|_{F}^2}{2}, \forall t \label{eq:10} \\
\mathbf{V}^* & = arg \min_{\mathbf{V} \in \mathbb{R}^{P \times r}} \sum_{t=1}^N h_t(\mathbf{a}_t^*, \mathbf{V}) + \lambda \frac{\|\mathbf{V}\|_{F}^2}{2} \label{eq:10-2}
\end{align}

\subsection{Sequential Logistic PCA (SLPCA)}

For a sequential version of BLPCA, $\mathbf{V}$ is of fixed dimension when data is streaming in. However, the dimension of $\mathbf{A}$ would change after every step. Similar to~\cite{mardani2013rank}, at each time $t$, we solve a local sub-optimal for the $t^{th}$ row of $\mathbf{A}$ (i.e. $\widetilde{\mathbf{a}}_t$) instead of a global one, and sequentially update $\mathbf{V}$ with the $\widetilde{\mathbf{a}}_t$'s (i.e. $\widetilde{\mathbf{V}}^t$). At step $t$, this means that we solve for $\widetilde{\mathbf{a}}_t$ and $\widetilde{\mathbf{V}}^t$ based on the best sub-optimal solution at step $t-1$. We call this Sequential LPCA (SLPCA):
\begin{align}
\widetilde{\mathbf{a}}_t & = arg \min_{\mathbf{a} \in \mathbb{R}^{1 \times r}} h_t(\mathbf{a}, \widetilde{\mathbf{V}}^{t-1}) + \gamma \frac{\|\mathbf{a}\|_{F}^2}{2}, \forall t \label{eq:11} \\
\widetilde{\mathbf{V}}^t & = arg \min_{\mathbf{V} \in \mathbb{R}^{P \times r}} \sum_{s=1}^t h_s(\widetilde{\mathbf{a}}_s, \mathbf{V}) + \lambda \frac{\|\mathbf{V}\|_{F}^2}{2} \label{eq:11-2}
\end{align}
Equation \eqref{eq:11} is easy to solve with a Newton method based gradient descent algorithm, since it is only a vector and the target function is strictly convex. Equation \eqref{eq:11-2} can be solved sequentially based on the past value $\widetilde{\mathbf{V}}^{t-1}$. To see how this works, we define a surrogate function $\widetilde{h}_t(\mathbf{a}_t, \mathbf{V})$ to approximate $h_t(\mathbf{a}_t, \mathbf{V})$:
\begin{align}\label{eq:12}
\widetilde{h}_t(\widetilde{\mathbf{a}}_t, \mathbf{V}) & = h_t(\widetilde{\mathbf{a}}_t, \widetilde{\mathbf{V}}^{t-1}) + \nabla_{\mathbf{V}} h_t(\widetilde{\mathbf{a}}_t, \widetilde{\mathbf{V}}^{t-1})^T ( \mathbf{V} - \widetilde{\mathbf{V}}^{t-1} ) \nonumber \\
& + \frac{\alpha_t}{2} \| \mathbf{V} - \widetilde{\mathbf{V}}^{t-1} \|_{F}^2, \ \ \alpha_t \geq \| \nabla_{\mathbf{V}}^2 h_t \|_{opt}
\end{align}
where $\| \cdot \|_{opt}$ is the operator norm. 
From the above it follows that $\widetilde{h}_t(\widetilde{\mathbf{a}}_t, \mathbf{V}) \geq h_t(\widetilde{\mathbf{a}}_t, \widetilde{\mathbf{V}}^{t-1})$, and moreover, as we solve Equation \eqref{eq:11-2} under $\widetilde{h}_t$ instead of $h_t$, we get:
\begin{equation}\label{eq:13}
\widetilde{\mathbf{V}}^t = \widetilde{\mathbf{V}}^{t-1} - \eta_t \nabla_{\mathbf{V}} h_t(\widetilde{\mathbf{a}}_t, \widetilde{\mathbf{V}}^{t-1})
\end{equation}
where $\eta_t \propto (\sum_{\tau=1}^t \alpha_{\tau})^{-1}$ is the step size. The choice of step size $\eta_t$ deserves some discussions. We will investigate in Section III on the convergence of this algorithm w.r.t. the BLPCA result. The full SLPCA algorithms is shown below.

\begin{algorithm}
\Begin{
 Input: data $\mathbf{X} \in \mathbb{R}^{N \times P}$, $\mathbf{X}^*=2\mathbf{X}-1 \in \{-1,1\}$\;
 Initialization: $\widetilde{\mathbf{V}}^t \approx 0, C, \gamma, \epsilon, \beta \in (0,1), alpha$\;
 \For{$t = 1, \dots, N$, $l_t(\widetilde{\mathbf{a}}_t) \doteq h_t(\widetilde{\mathbf{a}}_t, \widetilde{\mathbf{a}}^{t-1}) + \lambda \frac{\|\widetilde{\mathbf{a}}_t\|_{F}^2}{2}$}{
  	Initialize $\widetilde{\mathbf{a}}_t = 0$\;
  	Initialize $\Delta = \nabla l_t(\widetilde{\mathbf{a}}_t) \left( \nabla^2 l_t(\widetilde{\mathbf{a}}_t) \right)^{-1} \nabla l_t(\widetilde{\mathbf{a}}_t)$\;
  	\While{$\lambda > \epsilon$}{
  		Let $\Delta = - \left( \nabla^2 l_t(\widetilde{\mathbf{a}}_t) \right)^{-1} \nabla l_t(\widetilde{\mathbf{a}}_t)$, $d=d_0$\;
  		\While{$\nabla l_t(\widetilde{\mathbf{a}}_t + d\Delta) > \nabla l_t(\widetilde{\mathbf{a}}_t) + \alpha d \nabla l_t(\widetilde{\mathbf{a}}_t)^T \Delta$}{
  			Update $d = \beta d$\;
  		}
  		Update $\widetilde{\mathbf{a}}_t = \widetilde{\mathbf{a}}_t + d\Delta$\;
  		Update $\Delta = \nabla l_t(\widetilde{\mathbf{a}}_t) \left( \nabla^2 l_t(\widetilde{\mathbf{a}}_t) \right)^{-1} \nabla l_t(\widetilde{\mathbf{a}}_t)$\;
  	}
  	Set $\eta_t$\;
  	Update $\widetilde{\mathbf{V}}^t = \widetilde{\mathbf{V}}^{t-1} - \eta_t \nabla_{\mathbf{V}} h_t(\widetilde{\mathbf{a}}_t, \widetilde{\mathbf{V}}^{t-1})$
 }
 }
 \
 \caption{Sequential LPCA (SLPCA) Pseudo-Code}
 \label{SLPCA}
\end{algorithm}

\section{Convergence Analysis}
In this section, we will discuss the convergence property of the SLPCA algorithm. Since our focus is mainly on developing this algorithm for binary data, we will keep our analysis on the Bernoulli random variable, and the \emph{loss function} is:
\begin{equation}\label{eq:14}
h_t(\mathbf{a}_t, \mathbf{V}) = \sum_j \log \left( 1 + \exp (-x_{tj}^* \mathbf{a}_t \mathbf{v}_j^T) \right)
\end{equation}
where $\mathbf{v}_j$ is the $j^{th}$ row of $\mathbf{V}$. It is worthy noted that the similar algorithm can be developed in other exponential family random variables.

\subsection{Evaluation Functions}
To evaluate BLPCA and SLPCA, we define three important functions that we want to study. 
\begin{align}\label{eq:15}
C_N (\mathbf{V}^*) & = \frac{1}{N} \sum_{t=1}^N h_t(\mathbf{a}_t^*, \mathbf{V}^*) \nonumber \\
\widehat{C}_N (\widetilde{\mathbf{V}}^N) & = \frac{1}{N} \sum_{t=1}^N h_t(\widetilde{\mathbf{a}}_t, \widetilde{\mathbf{V}}^N) \\
\widetilde{C}_N (\widetilde{\mathbf{V}}^N) & = \frac{1}{N} \sum_{t=1}^N \widetilde{h}_t(\widetilde{\mathbf{a}}_t, \widetilde{\mathbf{V}}^N) \nonumber
\end{align}
where $C_N (\mathbf{V}^*)$ is the average batch loss function; $\widehat{C}_N (\widetilde{\mathbf{V}}^N)$ is the average sequential loss function; $\widetilde{C}_N (\widetilde{\mathbf{V}}^N)$ is the average sequential loss function under the surrogate function. In~\cite{mardani2013rank} we have the relationship:
\begin{equation}\label{eq:16}
C_N (\mathbf{V}^*) \leq \widehat{C}_N (\widetilde{\mathbf{V}}^N) \leq \widetilde{C}_N (\widetilde{\mathbf{V}}^N)
\end{equation}
In online learning, \emph{Regret} is also of interest~\cite{zinkevich2003online}~\cite{shalev2011online}. \emph{Regret} takes locally best solution in every step, defined as:
\begin{equation}\label{eq:17}
\widehat{Re}_N = \frac{1}{N} \sum_{t=1}^N h_t(\widetilde{\mathbf{a}}_t, \widetilde{\mathbf{V}}^t)
\end{equation}

\subsection{Convergence Analysis}
\begin{lemma}
For $t=1,\cdots, N$ and $h_t (\cdot)$ defined in \eqref{eq:14}, $\|\nabla_{\mathbf{V}} h_t\|_{F} \leq \|\mathbf{a}\|_{F}$, and $\|\nabla_{\mathbf{V}}^2 h_t\|_{opt} \leq \frac{1}{4} \|\mathbf{a}\|_{F}^2$.
\end{lemma}
\begin{proof}
For $h_t(\mathbf{a}_t, \mathbf{V})$, w.l.o.g., let rank$(\mathbf{\Theta})=1$, we have:
\begin{align}\label{eq:18}
\left[ \nabla_{\mathbf{V}} h_t \right]_j & = -\frac{x_{tj}^* \mathbf{a}_t}{1 + \exp (x_{tj}^* \mathbf{a}_t \mathbf{v}_j^T)} \nonumber \\
\left[ \nabla_{\mathbf{V}}^2 h_t \right]_{ij} & = \left( \frac{x_{tj}^* \mathbf{a}_t \delta_{ij}}{2 \cosh (\frac{1}{2} x_{tj}^* \mathbf{a}_t \mathbf{v}_j^T )} \right)^2 \nonumber
\end{align}
where $\delta_{ij}=1$ only when $i=j$ means matrix $\nabla_{\mathbf{V}}^2 h_t$ is diagonal. Since $\cosh(x) \geq 1$, hence the norms satisfy $\|\nabla_{\mathbf{V}} h_t\|_{F} \leq \|\mathbf{a}\|_{F}$, and $\|\nabla_{\mathbf{V}}^2 h_t\|_{opt} \leq \frac{1}{4} \|\mathbf{a}\|_{F}^2$.  
\end{proof}
\begin{lemma}
Let $\widetilde{\mathbf{a}}_t$ be bounded by $\Omega$, for $\forall t = 1,\cdots, N$. Based on \eqref{eq:13} we have $\|\widetilde{\mathbf{V}}^t - \widetilde{\mathbf{V}}^{t-1}\|_{F} \leq \eta_t \Omega$.
\end{lemma}
\begin{proof}
From Equation \eqref{eq:13}, we have $\|\widetilde{\mathbf{V}}^t - \widetilde{\mathbf{V}}^{t-1}\|_{F} = \eta_t \| \nabla_{\mathbf{V}} h_t \|_{F}$. Since $\widetilde{\mathbf{a}}_t$ result from a regularized problem in \eqref{eq:11}, so $\widetilde{\mathbf{a}}_t$ is bounded by $\Omega$. Thus we have $\|\widetilde{\mathbf{V}}^t - \widetilde{\mathbf{V}}^{t-1}\|_{F} \leq \eta_t \| \widetilde{\mathbf{a}}_t \|_{F} \leq \eta_t \Omega$.
\end{proof}
\begin{lemma}
For $h_t(\cdot)$ in Equation \eqref{eq:14}. $\langle \mathbf{a}, \nabla_{\mathbf{a}} h_t \rangle = \langle \mathbf{V}, \nabla_{\mathbf{V}} h_t \rangle$. Hence, for $t=1,\cdots, N$, $\eta_t \gamma \|\widetilde{\mathbf{a}}_t\|_{F}^2 = \langle \widetilde{\mathbf{V}}^{t-1}, -\eta_t \nabla_{\mathbf{V}} h_t \rangle = \langle \widetilde{\mathbf{V}}^{t-1}, \widetilde{\mathbf{V}}^t - \widetilde{\mathbf{V}}^{t-1} \rangle$.
\end{lemma}
This follows directly from \eqref{eq:11} and \eqref{eq:13}. 
\begin{lemma}
$h_t(\cdot)$ and surrogate function $\widetilde{h}_t(\cdot)$, as well as their first derivative $\nabla h_t(\cdot)$ and $\nabla \widetilde{h}_t(\cdot)$ are all Lipschitz continuous.
\end{lemma}
This is indicated directly from Lemma 1 \& Lemma 2 and the definition of Lipschitz continuous~\cite{bertsekas1999nonlinear}. 
\begin{theorem}[Proposition 2,~\cite{mardani2013rank}]
Under the regularity condition of Lemma 4, and $h_t(\cdot)$ a convex function, $\widetilde{C}_N (\widetilde{\mathbf{V}}^N)$ converges a.s. to $C_N (\mathbf{V}^*)$. Thus, from \eqref{eq:16} we directly see that $\widehat{C}_N (\widetilde{\mathbf{V}}^N)$ converges a.s. to $C_N (\mathbf{V}^*)$.
\end{theorem}
The Proof can be found in~\cite{mairal2010online} and~\cite{mardani2013rank}, following a quasi-martingale theory. 
\begin{theorem}
Given step size as $\eta_t=Ct^{-1/2}$ or $\eta_t=C$, the Regret $\widehat{Re}_N$ converges to within a constant of $\widehat{C}_N (\widetilde{\mathbf{V}}^N)$, and thus converges to within a constant of $C_N (\mathbf{V}^*)$.
\end{theorem}
A sketch of proof is given in Appendix A. The results basically show that $\lim_{N \to \infty} | \widehat{Re}_N - \widehat{C}_N (\widetilde{\mathbf{V}}^N)| \leq \frac{\gamma \Omega^2}{2}$ if $\eta_t=C t^{-1/2}$ and $\lim_{N \to \infty} | \widehat{Re}_N - \widehat{C}_N (\widetilde{\mathbf{V}}^N)| \leq \gamma \Omega^2 + C \Omega^2$ if $\eta_t=C$, $\Omega$ as a constant. From Theorem 1 \& Theorem 2, we recognize that both the average sequential function and Regret function converge to within a constant from the average batch optimum. 

However, it should be noted here that a better convergence result could be possible, probably by re-design the algorithms, which is one of our future tasks.

\section{Experimental Results}
\subsection{Simulated Binary-State System}
Firstly, we use simulated binary data to test the performance of our SLPCA algorithm in binary-state system. The generation of correlated Bernoulli sequences is illustrated in~\cite{lunn1998note}. In this work, we focus on the case where rank$(\mathbf{\Theta})=1$ since this usually demonstrates the best dimension reduction capability. It should be noted here that the extension to multiple Principal Components is straight-forward following the iterative updating rules in~\cite{collins2001generalization}. 

\begin{figure}[ht]
\begin{subfigure}{0.5\textwidth}
	\centering
	\includegraphics[trim = 0mm 0mm 0mm 5mm, width=1\columnwidth]{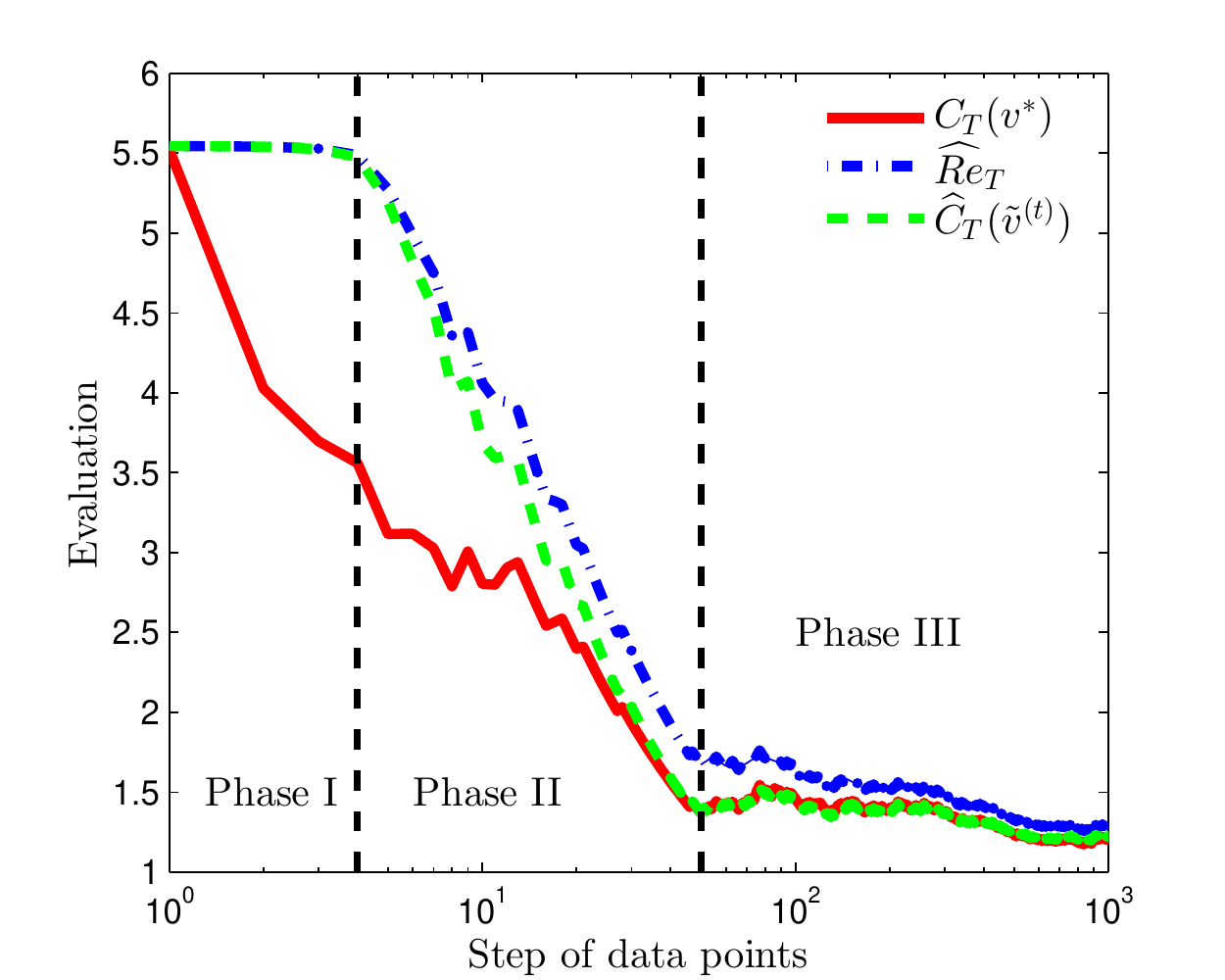}
\end{subfigure}
\begin{subfigure}{0.5\textwidth}
	\centering
	\includegraphics[trim = 0mm 0mm 0mm 5mm, width=1\columnwidth]{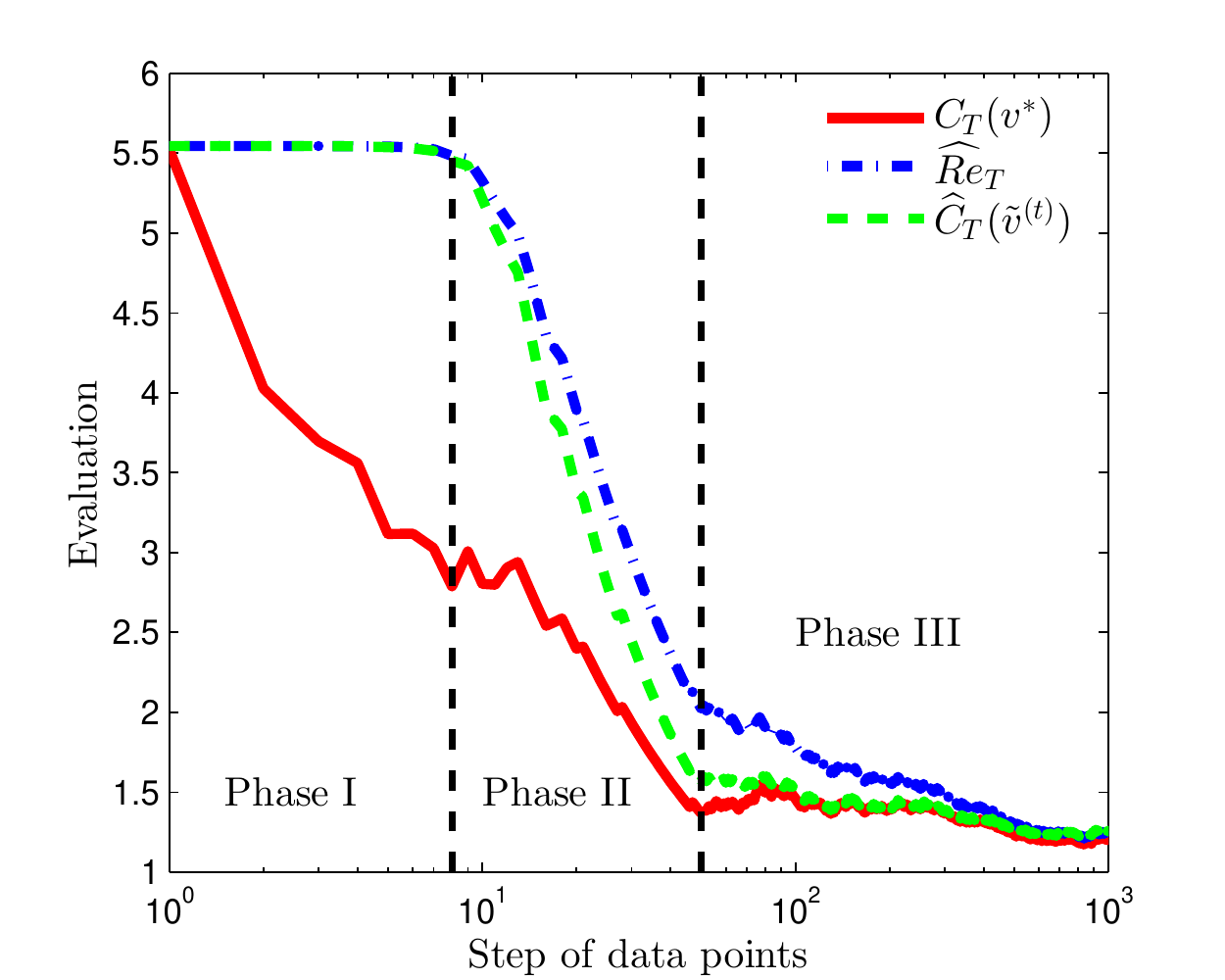}
\end{subfigure}
\caption{The three functions $C_t (\mathbf{V}^*)$, $\widehat{C}_t (\widetilde{\mathbf{V}}^t)$ and $\widehat{Re}_t$ as function of $t$. Top: $\eta_t = C t^{-1/2}$, with $C = 0.2$, $\gamma = 0.1$. Bottom: $\eta_t = C$, with $C = 0.05$, $\gamma = 0.1$.}
\label{EvalSim}
\end{figure}

\begin{figure*}[ht]
\begin{subfigure}{1\textwidth}
	\centering
	\includegraphics[trim = 0mm 0mm 0mm 5mm, width=1\textwidth, height=5cm]{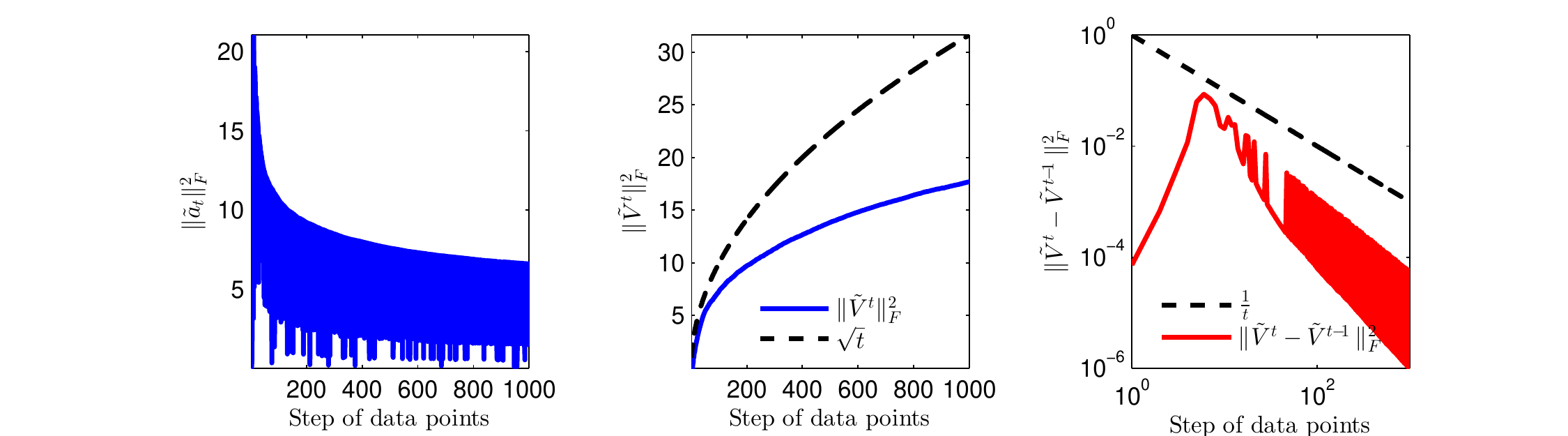}
\end{subfigure}
\begin{subfigure}{1\textwidth}
	\centering
	\includegraphics[trim = 0mm 0mm 0mm 2.5mm, width=1\textwidth, height=5cm]{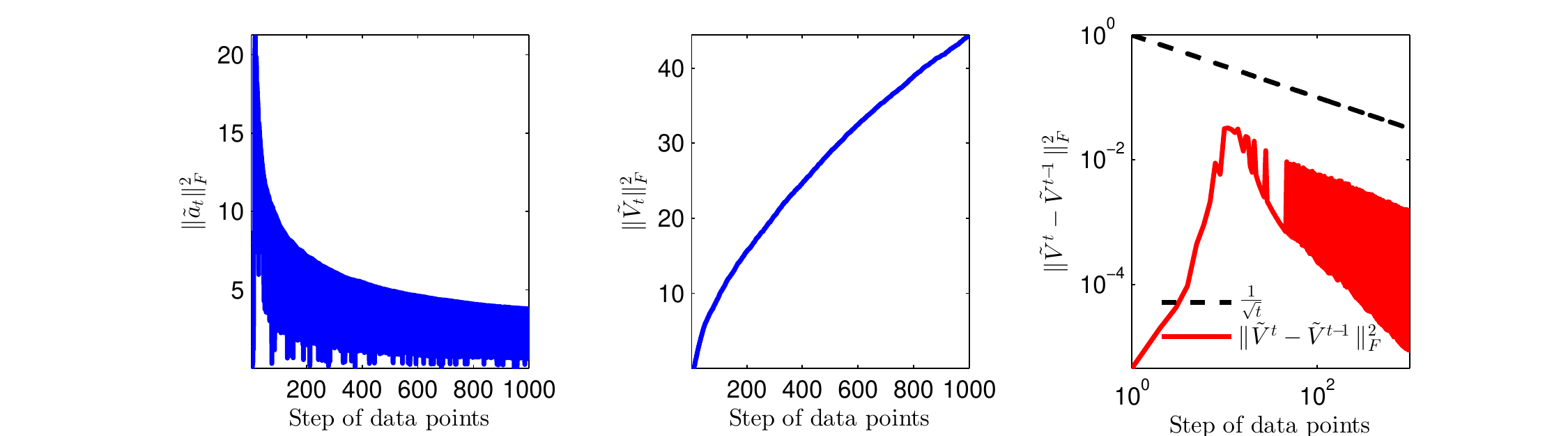}
\end{subfigure}
\caption{The convergence property of $\widetilde{\mathbf{a}}_t$, $\widetilde{\mathbf{V}}^t$ and $\|\widetilde{\mathbf{V}}^t - \widetilde{\mathbf{V}}^{t-1}\|_{F}$. Top: $\eta_t=Ct^{-1/2}$, with $C = 0.2$, $\gamma = 0.1$. Bottom: $\eta_t=C$, with $C = 0.05$, $\gamma = 0.1$.}
\label{ConvSim}
\end{figure*}

We tried the above on data with $P = 8$ dimension and length of $N = 1000$ data points. We initialize $\widetilde{\mathbf{V}}^0$ such that its norm is close but not equal to zero, for computation and convergence purposes. Fig~\ref{EvalSim} shows the three functions defined in \eqref{eq:15}; whereas Fig~\ref{ConvSim} shows the key parameters in the sequential steps. There are some interesting findings. 

Firstly, though both $\widehat{C}_N (\widetilde{\mathbf{V}}^N)$ and $\widehat{Re}_N$ converges at least within a constant to $C_N (\mathbf{V}^*)$, the stochastic learning can be clearly divided into three Phases, as shown in Fig~\ref{EvalSim}. Phase I stands for the period when the norm of $\widetilde{\mathbf{V}}^0$ is close to zero right after the initialization, when $h_t(\mathbf{a}_t, \mathbf{V})$ approaches $P\log2$ as in Equation \eqref{eq:14}. Phase II characterizes the decay of error versus $N$, whereas Phase III stands for when the error converges to within a constant independent of $N$. 

Secondly, $\|\widetilde{\mathbf{V}}^t\|_{F}^2$ increases versus $t$, which means that $\|\widetilde{\mathbf{V}}^t\|_{F}^2$ behaves differently from the coefficient in sequential learning of linear model~\cite{mairal2010online}~\cite{mardani2013rank}. Matrix factorization places no constraints for $\widetilde{\mathbf{V}}^t$, hence cannot guarantee the bound of $\widetilde{\mathbf{V}}^t$. From another perspective, $\widetilde{\mathbf{a}}_t$ is bounded since Equation \eqref{eq:11} has consistent regularization, while $\widetilde{\mathbf{V}}_t$ not since there is a summation of loss functions in Equation \eqref{eq:11-2}. It should be noted that, in Fig~\ref{ConvSim}, $\widetilde{\mathbf{a}}_t$ decreases versus $t$, which could result from \eqref{eq:11} and is an interesting topic in the future. 

Thirdly, due to the unbounded $\widetilde{\mathbf{V}}^t$, the term $\|\widetilde{\mathbf{V}}^t - \widetilde{\mathbf{V}}^{t-1}\|_{F}$ is not $\propto t^{-1}$ as in~\cite{mairal2010online} and~\cite{mardani2013rank}. It should be noted that the theoretical bound for $\|\widetilde{\mathbf{V}}^t - \widetilde{\mathbf{V}}^{t-1}\|_{F}$ under constant step size could be as low as $t^{-1/2}$, which could be a result of the convergence behavior of $\widetilde{\mathbf{a}}_t$ under constant step size. 

Last but not least, it is important to mention that the bounds obtained in Theorem 2 assume $N$ large enough. However, in many cases the decay of $N$ is not that fast. Therefore, the effect of $N$ cannot be completely ignored in the analysis.

\subsection{Building End-Use Energy Modeling}
Here, we introduce an application of SLPCA in Building Energy End-Use Modeling. Building End-Uses corresponds to the energy sectors that are occupant-driven. This subject has attracted significant interest in recent years because building energy shows strong dependence on end-user behavior, e.g. plug-in loads, user-controlled lighting, user-adjusted HVAC, etc.~\cite{swan2009modeling}~\cite{kang2014modeling}. 

Energy end-use modeling has been attempted from either a top-down or a bottom-up approach. In this work, since we are more interested in modeling occupant behavior, we adopt the bottom-up approach. This approach is usually based on stochastic simulations of the energy usage pattern for each individual appliance. Dimensional reduction can help to generate one or more \emph{Principal Appliances}, and can more efficiently characterize the whole space energy consumption.

\begin{figure}[ht]
\centering
\includegraphics[trim = 0mm 0mm 0mm 5mm, width=1\columnwidth]{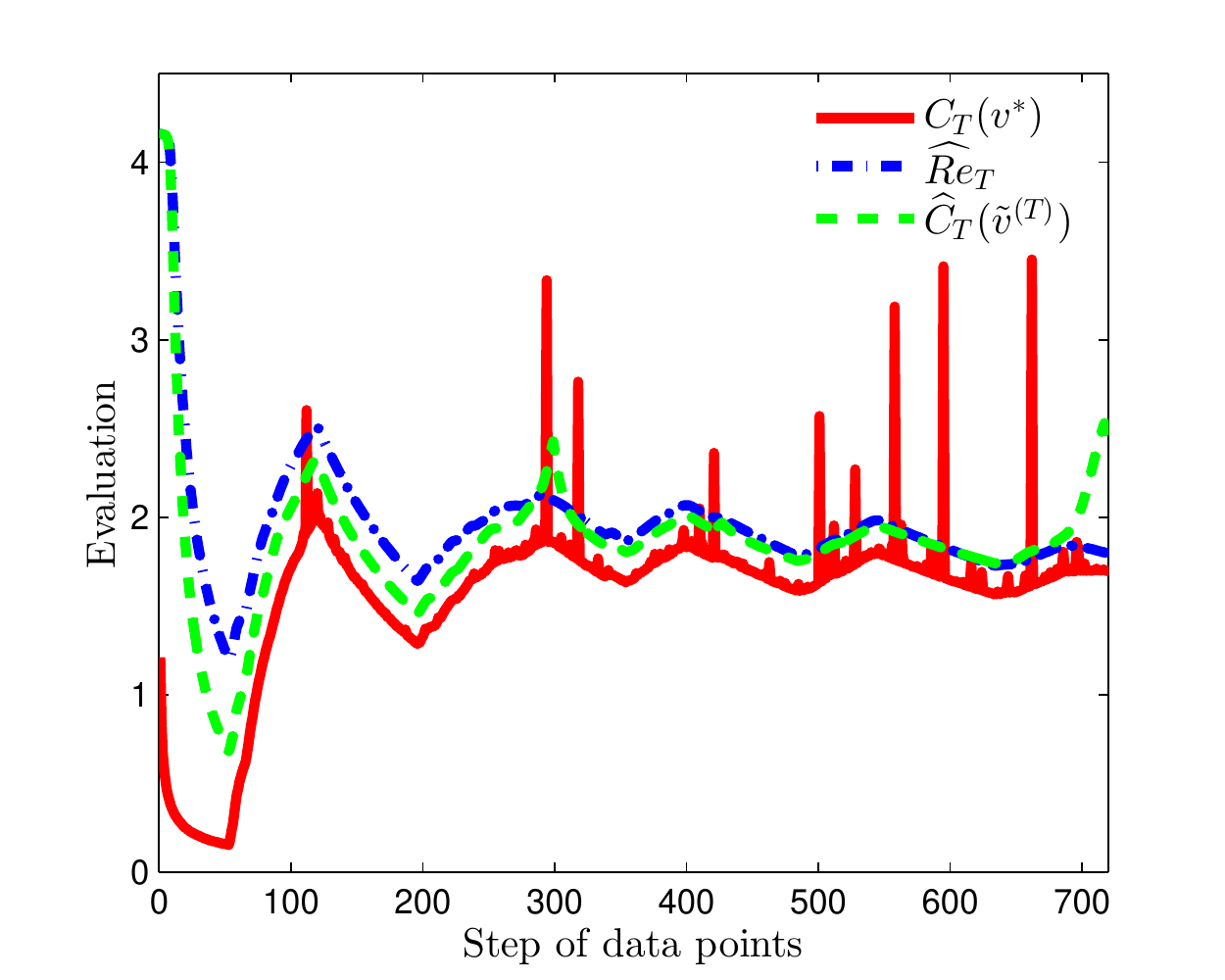}
\caption{The three functions $C_t (\mathbf{V}^*)$, $\widehat{C}_t (\widetilde{\mathbf{V}}^t)$ and $\widehat{Re}_t$ as function of $t$ for energy end-use simulation with constant step size $\eta_t = C$ as $C = 0.05$, $\gamma = 0.1$.}
\label{EvalEU}
\end{figure}

Here, we want to study the modeling of all the computer monitors in a small, shared work space. We collect the data of 6 monitors in 10 minutes interval, and use BLPCA and SLPCA to obtain the \emph{Principal Monitor} profile of the building. Considering that the pattern could be non-stationary, we choose the constant step size that is short enough to track the changes as they appear\footnote{one could presumably also leverage the likely periodic behavior of the data by appropriate aggregation}. We also only consider the first \emph{Principal Monitor} to achieve the best dimensional reduction. The convergence of the algorithm is shown in Fig~\ref{EvalEU}. We observe a good convergence for both $\widehat{C}_N (\widetilde{\mathbf{V}}^N)$ and $\widehat{Re}_N$. Periodic fluctuation is observed, due to the periodic transition between day and night energy consumption, which results in periodical changing of the data model. Moreover, the online algorithm demonstrate less fluctuations because they adaptively update the model of the data. 

The BLPCA, SLPCA and Regret are used to reconstruct the original data, as illustrated in Section II-C, when we discussed the reconstruction of $\mathbf{X}$ by $g(\mathbf{\Theta})$. The results are compared with the original data in Fig~\ref{ConvEU} (sum of states of all appliances, 1 as ON and 0 as OFF). Interestingly, Regret gives better approximation to BLPCA since it uses locally best pairs of $\widetilde{\mathbf{a}}_t$ and $\widetilde{\mathbf{V}}^t$, so that can better catch the periodic pattern of the original data. Whereas SLPCA uses the $\widetilde{\mathbf{V}}^T$, which could probably give unpromising result if data is non-stationary. 

\begin{figure}[ht]
\centering
\includegraphics[trim = 10mm 10mm 10mm 10mm, width=1\columnwidth]{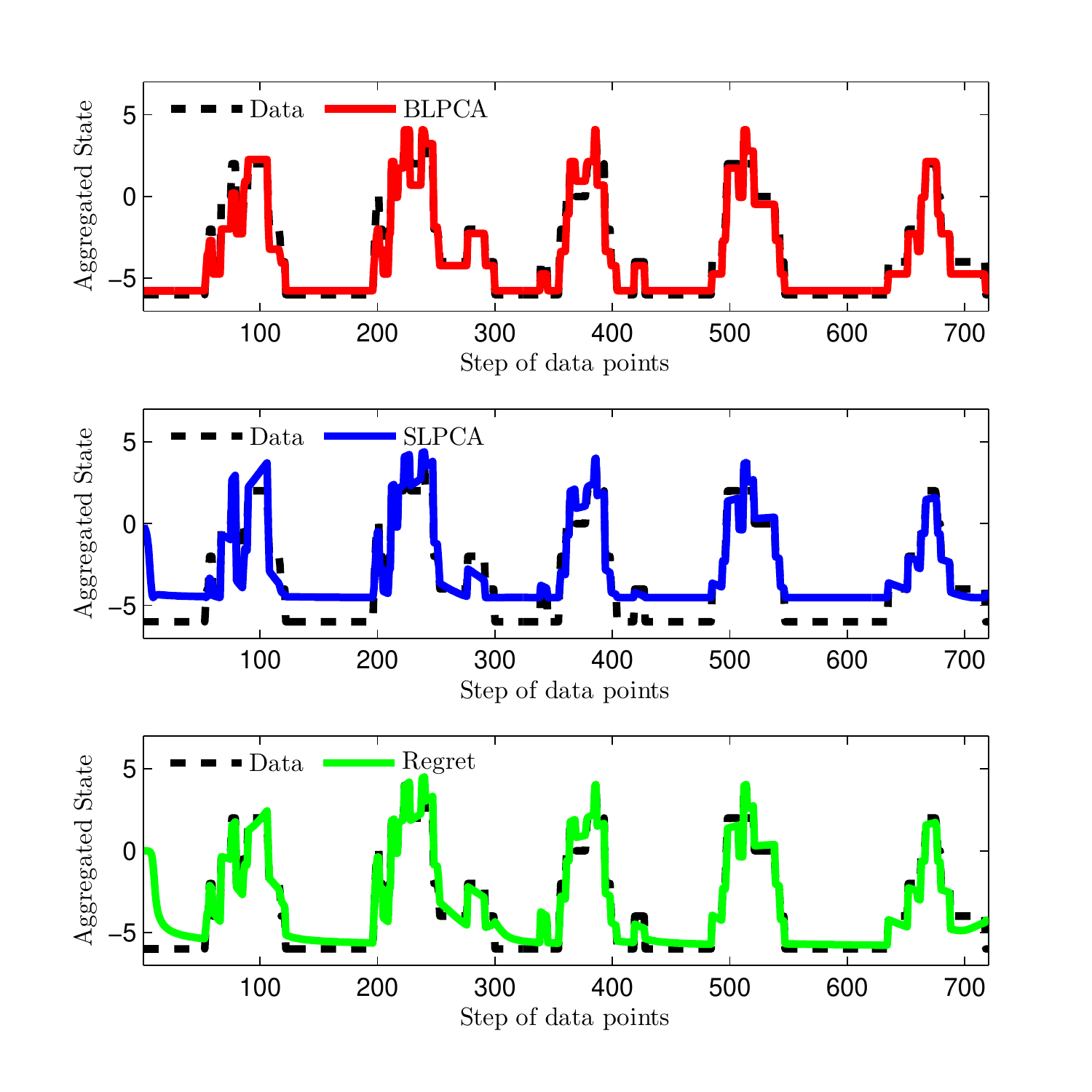}
\caption{Reconstruction of the aggregated state (sum of states of 6 monitors) under the three functions $C_t (\mathbf{V}^*)$, $\widehat{C}_t (\widetilde{\mathbf{V}}^t)$ and $\widehat{Re}_t$ as function of $t$.}
\label{ConvEU}
\end{figure}

\section{Conclusion}
Sequential or online dimension reduction addresses more and more attentions due to the explosion of streaming data based application and the requirement of adaptive statistical modeling in many emerging fields. In this work, we extend the theory of \emph{e}PCA or LPCA to sequential version based on online convex optimization theory, which can maintain the capability to model large families of distributions, at the same time achieve the computation and storage efficiency. In our work, we define two functions to evaluate the SLPCA algorithm, the average sequential target function $\widehat{C}_N (\widetilde{\mathbf{V}}^N)$ and the Regret function $\widehat{Re}_N$, and show that both of them converge at least within a constant to BLPCA results. We also demonstrate an application of this algorithm in building energy end-use modeling. 

\appendix
\begin{lemma}
For $t=1,\cdots, N$, if $\Omega$ is the upper bound of $\|\mathbf{a}\|_{opt}^2$ as in Lemma 2, $\|\widetilde{\mathbf{V}}^t\|_{F}^2 \leq \Omega^2 \sum_{s=1}^t \eta_s^2 + 2\gamma \Omega^2 \sum_{s=1}^t \eta_s$.
\end{lemma}
\begin{proof}
We start from the relationship:
\begin{align}
\|\widetilde{\mathbf{V}}^t - \widetilde{\mathbf{V}}^{t-1}\|_{F}^2 & = \|\widetilde{\mathbf{V}}^t\|_{F}^2 - \|\widetilde{\mathbf{V}}^{t-1}\|_{F}^2 - 2 \langle \widetilde{\mathbf{V}}^t - \widetilde{\mathbf{V}}^{t-1}, \widetilde{\mathbf{V}}^{t-1} \rangle \nonumber \\
& = \|\widetilde{\mathbf{V}}^t\|_{F}^2 - \|\widetilde{\mathbf{V}}^{t-1}\|_{F}^2 - 2\eta_t \gamma \|\widetilde{\mathbf{a}}^t\|_{F}^2 \nonumber
\end{align}
We sum over the LHS and RHS and get:
\begin{equation}
\sum_{s=1}^t \|\widetilde{\mathbf{V}}^s - \widetilde{\mathbf{V}}^{s-1}\|_{F}^2 + 2 \gamma \sum_{s=1}^t \eta_s \|\widetilde{\mathbf{a}}_s\|_{F}^2 = \|\widetilde{\mathbf{V}}^t\|_{F}^2 - \|\widetilde{\mathbf{V}}^0\|_{F}^2 \nonumber
\end{equation}
For simplicity, assume $\|\widetilde{\mathbf{V}}^0\|_{F}^2 \approx 0$, we prove the lemma. 
\end{proof}

Now turn to proof of Theorem 2. Based on \eqref{eq:13} we have:
\begin{align}
\|\widetilde{\mathbf{V}}^t - \widetilde{\mathbf{V}}^N\|_{F}^2 & = \|\widetilde{\mathbf{V}}^{t-1} - \widetilde{\mathbf{V}}^N\|_{F}^2 + \eta_t^2 \|\nabla_{\mathbf{V}} h_t\|_{F}^2 \nonumber \\
& - 2\eta_t \langle \nabla_{\mathbf{V}} h_t, \widetilde{\mathbf{V}}^{t-1} - \widetilde{\mathbf{V}}^N \rangle \nonumber
\end{align}
From Lemma 1, Lemma 5, and $\|\nabla_{\mathbf{V}} h_t\|_{F}^2 \leq \Omega^2$, thus:
\begin{align}
& N \lbrace \widehat{Re}_N - \widehat{C}_N (\widetilde{\mathbf{V}}^N) \rbrace \leq \sum_{t=1}^N \langle \nabla_{\mathbf{V}} h_t, \widetilde{\mathbf{V}}^{t-1} - \widetilde{\mathbf{V}}^N \rangle \nonumber \\
& \leq \frac{\|\widetilde{\mathbf{V}}^N\|_{F}^2}{2\eta_0} + \sum_{t=1}^N \left( \frac{1}{2\eta_t} - \frac{1}{2\eta_{t-1}} \right) \|\widetilde{\mathbf{V}}^N - \widetilde{\mathbf{V}}^{t-1}\|_{F}^2 + \frac{\Omega^2}{2} \eta_t \nonumber \\
& \leq \frac{\|\widetilde{\mathbf{V}}^N\|_{F}^2}{2\eta_0} + \sum_{t=1}^N \left( \frac{1}{2\eta_t} - \frac{1}{2\eta_{t-1}} \right) \|\widetilde{\mathbf{V}}^N\|_{F}^2 + \frac{\Omega^2}{2} \eta_t \nonumber
\end{align}
\begin{itemize}
	\item diminishing step size $\eta_t = Ct^{-1/2}$. From Lemma 5, we have:
	\begin{align}
	| \widehat{Re}_N - \widehat{C}_N (\widetilde{\mathbf{V}}^N)| & \leq \frac{\Omega^2 C}{2} \frac{\log N}{N} + \frac{\Omega^2 C}{4} \frac{\log N}{\sqrt{N}} \nonumber \\
	& + \frac{\Omega^2 (2\gamma + C)}{2 \sqrt{N}} + \frac{\gamma \Omega^2}{2} \nonumber
	\end{align}
	Then $\lim_{N \to \infty} | \widehat{Re}_N - \widehat{C}_N (\widetilde{\mathbf{V}}^N)| \leq \frac{\gamma \Omega^2}{2}$. But with reasonable $N$, the term $\frac{\Omega^2 C \log N}{\sqrt{N}}$ will also be significant. Usually, small $C$ and $\gamma$ can force a lower error bound. However, small $\gamma$ can result in more steps in optimizing for $\widetilde{\mathbf{a}}_t$, whereas small $C$ would make the step size too small, which may not be a good choice if we want a fast decaying of the error bound.
	\item constant step size $\eta_t = C$:
    For constant step, we have:
    \begin{equation}
    | \widehat{Re}_N - \widehat{C}_N (\widetilde{\mathbf{V}}^N)| \leq \gamma \Omega^2 + \Omega^2 C \nonumber
    \end{equation}
    Similarly, we prefer small small $C$ and $\gamma$. The challenge of using small $C$ and $\gamma$ have already been discussed. 
\end{itemize}

\section*{Acknowledgment}
This research is funded by the Republic of Singapore’s National Research Foundation through a grant to the Berkeley Education Alliance for Research in Singapore (BEARS) for the Singapore-Berkeley Building Efficiency and Sustainability in the Tropics (SinBerBEST) Program. BEARS has been established by the University of California, Berkeley as a center for intellectual excellence in research and education in Singapore.



%

\bibliographystyle{IEEEtran}
\bibliography{IEEEabrv}



\end{document}